\documentclass{article}
\usepackage{epsfig}
\usepackage{amsmath}
\usepackage{amsthm}
\usepackage{amssymb}
\usepackage{amsfonts}
\usepackage{subfigure}

\newtheorem{theorem}{Theorem}
\newtheorem{lemma}{Lemma}

\newlength{\boxwidth}
\newcommand{\singlespacing}%
{\small\normalsize}

\newcommand{\ignore}[1]{}

\newcommand{\R}{\mathbb{R}}
\newcommand{\Z}{\mathbb{Z}}
\newcommand{\FF}{\mathbb{F}}

\begin{document}
\title{Optimal Query Complexity for Reconstructing Hypergraphs}
\date{}
\author{Nader H. Bshouty \\ Technion,
Israel\\bshouty@cs.technion.ac.il \and Hanna Mazzawi \\ Technion,
Israel\\ hanna@cs.technion.ac.il } \maketitle
\begin{abstract}
In this paper we consider the problem of reconstructing a hidden
weighted hypergraph of constant rank using additive queries. We
prove the following: Let $G$ be a weighted hidden hypergraph of
constant rank with~$n$ vertices and $m$ hyperedges. For any $m$
there exists a non-adaptive algorithm that finds the edges of the
graph and their weights using
$$
O\left(\frac{m\log n}{\log m}\right)
$$
additive queries. This solves the open problem in [S. Choi, J. H.
Kim. Optimal Query Complexity Bounds for Finding Graphs. {\em STOC},
749--758,~2008].

When the weights of the hypergraph are integers that are less than
$O(poly(n^d/m))$ where $d$ is the rank of the hypergraph (and
therefore for unweighted hypergraphs) there exists a non-adaptive
algorithm that finds the edges of the graph and their weights using
$$
O\left(\frac{m\log \frac{n^d}{m}}{\log m}\right).
$$
additive queries.

Using the information theoretic bound the above query complexities
are tight.
\end{abstract}
\section{Introduction}
In this paper we consider the following problem of reconstructing
weighted hypergraphs of constant rank\footnote{Sometimes called
dimension.} (the maximal size of a hyperedge) using additive
queries: Let $G=(V,E,w)$ be a weighted hidden hypergraph where
$E\subset 2^V$, $|e|$ is constant for all $e\in E$, $w:E\to \R$, and
$n$ is the number of vertices in $V$. Denote by $m$ the size of~$E$.
Suppose that the set of vertices $V$ is known and the set of edges
$E$ is unknown. Given a set of vertices $S\subseteq V$, an additive
query, $Q_G(S)$, returns the sum of weights in the sub-hypergraph
induced by~$S$. That is,
$$
Q_G(S) = \sum_{e\in E \cap 2^S} w(e).
$$
Our goal is to exactly reconstruct the set of edges using additive
queries.

One can distinguish between two types of algorithms to solve the
problem. Adaptive algorithms are algorithms that take into account
outcomes of previous queries while non-adaptive algorithms make
all queries in advance, before any answer is known. In this paper,
we consider non-adaptive algorithms for the problem. Our concern
is the query complexity, that is, the number of queries needed to
be asked in order to reconstruct the hypergraph.

The hypergraph reconstructing problem has known a significant
progress in the past decade. For unweighted hypergraph of rank $d$
the information theoretic lower bound gives $$\Omega
\left(\frac{m\log \frac{n^d}{m}}{\log m}\right)$$ for the query
complexity for any adaptive algorithm for this problem.

Many independent results \cite{D75,GK00,DR81,BG07} have proved a
tight upper bound for hypergraph of rank~$1$, i.e., loops. A tight
upper bound was proved for some subclasses of unweighted hypergraphs
of rank two, i,e., graphs (Hamiltonian graphs, matching, stars and
cliques etc.) \cite{G97,GK98,GK00,BGK}, unweighted graphs with
$\Omega(dn)$ edges where the degree of each vertex is bounded by $d$
\cite{GK00}, graphs with $\Omega(n^2)$ edges \cite{GK00} and then
the former was extended to $d$-degenerate unweighted graphs with
$\Omega(dn)$ edges \cite{G97}, i.e., graphs that their edges can be
changed to directed edges where the out-degree of each vertex is
bounded by~$d$. A recent paper by Choi and Kim, \cite{CH}, gave a
tight upper bound for all unweighted graphs. In this paper we give a
tight upper bound for all unweighted hypergraphs of constant rank.
Our bound is tight even for weighted hypergraphs with integer
weights $|w|= poly(n^d/m)$ where~$d$ is the rank of the hypergraph.

\begin{figure}\begin{center}
\begin{tabular}{l|c|c|c|}
                     & Tight Upper  & Adaptive   & Non-adaptive  \\
                     &  Bound & Poly. time  &   Poly. time\\
\hline {\bf Loops rank$=1$}\\ \hline

Unweighted Loops      &  \cite{D75,GK00,DR81,BG07}   &   \cite{BSH} &OPEN \\

\hline

Bounded Weighted Loops  & \cite{CH} & OPEN & OPEN \\

\hline

Unbounded Weighted Loops & \cite{BM2} & OPEN${}^\dag$ & OPEN${}^\S$ \\ \hline {\bf Graph rank$=2$}\\

\hline

Unweighted Graph      &  \cite{CH}   &   \cite{M10} &OPEN \\

\hline

Bounded Weighted Graph  & \cite{CH,BM1} & OPEN & OPEN \\

\hline

Unbounded Weighted Graph & \cite{BM2} & OPEN${}^\dag$ & OPEN \\ \hline {\bf Hypergraph rank$>2$}\\

\hline

Unweighted HyperGraph   & Ours & OPEN & OPEN \\

\hline

Unbounded Weighted Hypergraph & Ours & OPEN${}^\dag$ & OPEN \\

\hline

\end{tabular}\caption{Results for weighted and un-weighted hypergraphs
with optimal query complexity. ${}^\dag$A non-optimal adaptive
query complexity algorithm for Hypergraph can be found in
\cite{CJK}. ${}^\S$ A non-optimal non-adaptive query complexity
algorithms can be found in \cite{IR} and the references within
it.}
\end{center}
\end{figure}

\ignore{ Grebinski and Kucherov first proved in \cite{GK00} the
existence of a non-adaptive algorithm for reconstructing the class
of $d$-bounded degree unweighted graphs using $O(dn)$ queries.
They also show a polynomial time algorithm for reconstructing an
unweighted graph using $O(n^2/\log n)$ queries. In \cite{G97}
Grebinski proved the existence of a non-adaptive reconstructing
algorithm for $d$-degenerate unweighted graphs that uses $O(dn)$
queries. In \cite{RS07} Reyzin and Srivastava show a polynomial
time algorithm for reconstructing an unweighted graph using
$O(m\log n)$ queries. In \cite{CH}, S.~Choi and J. Han Kim prove
the existence of a non-adaptive algorithm for reconstructing any
unweighted graph using $O\left(\frac{m\log \frac{n^2}{m}}{\log
m}\right)$.}

For weighted hypergraph of constant rank with unbounded weights the
information theoretic lower bound gives $$ \Omega\left(\frac{m\log
n}{\log m}\right)
$$
In \cite{CH}, Choi and Kim prove a tight upper bound for loops
(hypergraph of rank~$1$). For weighted graphs (hypergraph of rank
$2$) Choi and Kim, \cite{CH}, proved the following: If $m>(\log
n)^\alpha$ for sufficiently large $\alpha$, then, there exists a
non-adaptive algorithm for reconstructing a weighted graph where the
weights are real numbers bounded between $n^{-a}$ and $n^b$ for any
positive constants $a$ and~$b$ using
$$
O\left(\frac{m\log n}{\log m}\right)
$$
queries.

In \cite{BM1}, Bshouty and Mazzawi close the gap in $m$ and proved
that for any weighted graph where the weights are bounded between
$n^{-a}$ and $n^b$ for any positive constants $a$ and $b$ and {\it
any} $m$ there exists a non-adaptive algorithm that reconstructs
the hidden graph using
$$
O\left(\frac{m \log n}{\log m}\right)
$$
queries. Then in \cite{BM2} they extended the result to any
weighted graph with any unbounded weights.

In this paper extend all the above results to any hypergraph of
constant rank, i.e., the edges of the graph has constant size. This
solves the open problems in~\cite{CH,BM1,BM2}.

The paper is organized as follows: In Section~2, we present
notation, basic tools and some background. In Section~3, we prove
the main result.

\section{Preliminaries}
In this section we present some background, basic tools and
notation.

For an integer $r$ let $[r]$ be the set $\{1,2,\ldots,r\}$. For
$S\subset [r]$ we define $x^S\in \{0,1\}^r$ where $x^S_i=1$ if and
only if $i\in S$. The inverse operation is $S^x=\{i\ |\ x_i=1\}$. We
say that $x_1,\ldots,x_d\in \{0,1\}^n$ are {\it pairwise disjoint}
if for every $i\not=j$, we have $x_i* x_j={\bf 0}$ where $*$ is
component-wise product of two vectors. For a prime $p$ and integers
$a$ and $b$ we write $a=_pb$ for $a=b\mod p$. We will also allow
$p=\infty$. In this case $a$ and $b$ can be any real numbers and
$a=_\infty b$ will mean $a=b$ as real numbers.

\subsection{$d$-Dimensional Matrices}
A $d$-{\it dimensional matrix} $A$ of {\it size}
$n_1\times\cdots\times n_d$ over a field $\FF$ is a map
$A:\prod_{i=1}^d[n_i]\to \FF$. We denote by
$\FF^{n_1\times\cdots\times n_d}$ the set of all $d$-dimensional
matrices $A$ of size $n_1\times\cdots\times n_d$. We write
$A_{i_1,\ldots,i_d}$ for $A(i_1,\ldots,i_d)$.

The zero map is denoted by $0^{n_1\times\cdots\times n_d}$. The
matrix $$B=(A_{i_1,i_2,\ldots,i_d})_{i_1\in I_1,i_2\in
I_2,\ldots,i_d\in I_d},$$ where $I_j\subseteq [n_j]$, is the
$|I_1|\times\cdots\times |I_d|$ matrix where
$B_{j_1,\ldots,j_d}=A_{\ell_1,\ldots,\ell_d}$ and $\ell_i$ is the
$j_i$th smallest number in $I_i$. When $I_j=[n_j]$ we just write $j$
and when $I_j=\{\ell\}$ we just write $j=\ell$. For example,
$(A_{i_1,i_2,\ldots,i_d})_{i_1,i_2=\ell,i_3\in I_2,\ldots,i_d\in
I_d}=(A_{i_1,i_2,\ldots,i_d})_{i_1\in [n_1],i_2\in\{\ell\},i_3\in
I_2,\ldots,i_d\in I_d}.$

When $n_1=n_2=\cdots=n_d=n$ then we denote
$\FF^{n_1\times\cdots\times n_d}$ by $\FF^{\times_dn}$ and
$0^{n_1\times\cdots\times n_d}$ by~$0^{\times_dn}$.

We say that the entry $A_{i_1,i_2,\ldots,i_d}$ is of {\it dimension}
$r$ if $|\{i_1,\ldots,i_d\}|=r.$ For $d$-dimensional matrix $A$ we
denote by $wt(A)$ the number of points in $\prod_{i=1}^d[n_i]$ that
are mapped to non-zero elements in $\FF$. We denote by $wt_r(A)$ the
number of points in $\prod_{i=1}^d[n_i]$ of dimension $r$ that are
mapped to non-zero elements in~$\FF$. Therefore,
$wt(A)=wt_1(A)+wt_2(A)+\cdots+wt_{d}(A).$ We denote by
${\mathcal{A}}_{d,m}$ the set of $d$-dimensional matrices $A\in
\FF^{\times_dn}$ where $wt_d(A)\le m$ and
${\mathcal{A}}^\star_{d,m}$ the set of $d$-dimensional matrices
$A\in \FF^{\times_dn}$ where $1\le wt_d(A)\le m$.

For $d$-dimensional matrix $A$ of size $n_1\times\cdots\times n_d$
and $x_i\in \FF^{n_i}$ we define
$$A(x_1,\ldots,x_d)=\sum_{i_1=1}^{n_1}\cdots \sum_{i_d=1}^{n_d}
A_{i_1,i_2,\ldots,i_d}x_{1i_1}\cdots x_{di_d}.$$ The vector
$v=A(\cdot,x_2,\ldots,x_d)$ is $n_1$-dimensional vector that its
$i$th entry is $$\sum_{i_2=1}^{n_2}\cdots \sum_{i_d=1}^{n_d}
A_{i,i_2,\ldots,i_d}x_{2i_2}\cdots x_{di_d}.$$ For a set of
$d$-dimensional matrices ${\mathcal{B}}$, a set $S\subseteq
(\{0,1\}^n)^d$ is called a {\it zero test set} for ${\mathcal{B}}$
if for every $A\in {\mathcal{B}}$, $A\not=0$, there is $x\in S$ such
that $A(x)\not=0$.

A $d$-dimensional matrix is called {\it symmetric} if for every
$i=(i_1,\ldots,i_d)\in [n]^d$ and any permutation $\phi$ on~$[d]$,
we have $A_i=A_{\phi i}$, where $\phi i =
(i_{\phi(1)},\ldots,i_{\phi(d)})$. Notice that for a symmetric
$d$-dimensional matrix $A\in \FF^{\times_d n}$, $x_i\in \{0,1\}^n$
and any permutation $\phi$ on $[d]$, we have
$A(x_1,\ldots,x_d)=A(x_{\phi(1)},\ldots,x_{\phi(d)}).$

We will be interested mainly in the fields $\FF=\R$ the field of
real numbers and $\FF=\Z_p$ the field of integers modulo~$p$ and in
matrices of constant $d=O(1)$ dimension. Also $p>d!$. Although it
seems that we are restricting the parameters, the final result has
no restriction on the parameters except for $d=O(1)$. We will also
abuse the notations $\Z_p$ and $=_p$ and allow $p=\infty$ (so in
this paper $\infty$ is also prime number). In that case
$\Z_\infty=\R$ and $=_\infty$ is equality in the filed of real
numbers.

\subsection{Hypergraph}
A {\it hypergraph} $G$ is a pair $G = (V,E)$ where $V=[n]$ is a set
of elements, called nodes or vertices, and $E$ is a set of non-empty
subsets of $2^V$ called {\it hyperedges} or {\it edges}. The {\it
rank} $r(G)$ of a hypergraph $G$ is the maximum cardinality of any
of the edges in the hypergraph. A hypergraph is called $d$-{\it
uniform} if all of its edges are of size~$d$.

A {\it weighted hypergraph} $G= (V,E,w)$ over $\Z_p$ is a
hypergraph $(V,E)$ with a weight function $w:E\to \Z_p$. For two
weighted hypergraph $G_1=(V,E_1,w_1)$ and $G_2=(V,E_2,w_2)$ we
define the weighted hypergraph $G_1-G_2=(V,E,w)$ where $E=\{ e\in
E_1\cup E_2\ | \ w_1(e)\not=w_2(e)\},$ and for every $e\in E$,
$w(e)=w_1(e)-w_2(e).$ Obviously, $G_1=G_2$ if and only if
$G_1-G_2$ is an independent set, i.e., $E=\emptyset$.

We denote by ${\mathcal{G}}_d$ the set of all weighted hypergraphs
over $\Z_p$ of rank at most~$d$, ${\mathcal{G}}_{d,m}$ the set of
all weighted hypergraphs over $\Z_p$ of rank at most $d$ and at most
$m$ edges and ${\mathcal{G}}^\star_{d,m}$ the set of all weighted
hypergraphs over $\Z_p$ of rank $d$ and at most $m$ edges.

Let $w^\star:2^V\to \Z_p$ be $w$ extended to all possible edges
where for $e\in E$, $w^\star(e)=w(e)$ and for $e\not\in E$,
$w^\star(e)=0$.

An {\it adjacency $d$-dimensional matrix of a weighted hypergraph}
 $G$ is a $d$-dimensional matrix $A_d^G$ where $d\ge r(G)$ such that
for every set $e=\{i_1,i_2,\ldots,i_\ell\}$ of size at most $d$ we
have $A^G_{d( j_1,\ldots,j_d)}=_pw^\star(e)/N(d,\ell)$ for all
$j_1,\ldots,j_d$ such that
$\{j_1,j_2,\ldots,j_d\}=\{i_1,\ldots,i_\ell\}$ where
$$N(d,\ell)=\sum_{i=0}^\ell (-1)^i{\ell\choose i}(\ell-i)^d.$$
That is, $N(d,\ell)$ is the number of possible sequences
$(j_1,\ldots,j_d)$ such that
$\{j_1,\ldots,j_d\}=\{i_1,\ldots,i_\ell\}$. Note that
$N(d,\ell)\le d!<p$ and therefore $N(d,\ell)\not=_p 0$ and $A^G_d$
is well defined.

It is easy to see that the adjacency matrix of a weighted hypergraph
is a symmetric matrix and $r(G)=r$ if and only if the adjacency
matrix of $G$ has an non-zero entry of dimension $r$ and all entries
of dimension greater than~$r$ are zero.

\subsection{Additive Model}
In the Additive Model the goal is to exactly learn a hidden
hypergraph with minimal number of additive queries. Given a set of
vertices $S\subseteq V$, an {\it additive query},~$Q_G(S)$,
returns the sum of weights in the subgraph induces by~$S$. That
is,
$
Q_G(S) =_p \sum_{e\in E \cap 2^S} w(e).
$
Our goal is to exactly reconstruct the set of edges and find their
weights using additive queries. See the many applications of this
problem in \cite{BGK,CH,CJK}.

We say that the set ${\mathcal{S}}=\{S_1,S_2,\cdots,S_k\}\subseteq
2^V$ is a {\it detecting set} for ${\mathcal{G}}_{d,m}$ if for any
hypergraph $G\in{\mathcal{G}}_{d,m}$ there is $S_i$ such that
$Q_G(S_i)\not=0$. We say that the set
${\mathcal{S}}=\{S_1,S_2,\cdots,S_k\}\subseteq 2^V$ is a {\it search
set} for ${\mathcal{G}}_{d,m}$ if for any two distinct hypergraphs
$G_1,G_2\in{\mathcal{G}}_{d,m}$ there is $S_i$ such that
$Q_{G_1}(S_i)\not=Q_{G_2}(S_i)$. That is, given $(Q_G(S_i))_i$ one
can uniquely determines $G$. We now prove the following,

\begin{lemma} \label{dtos}If ${\mathcal{S}}=\{S_1,S_2,\cdots,S_k\}\subseteq
2^V$ is a detecting set for ${\mathcal{G}}_{d,2m}$ then it is a
search set for ${\mathcal{G}}_{d,m}$.
\end{lemma}
\begin{proof} Let $G_1,G_2\in {\mathcal{G}}_{d,m}$ be two distinct weighted hypergraphs.
Let $G=G_1-G_2$. Since $G\in{\mathcal{G}}_{d,2m}$ there must be
$S_i\in {\mathcal{S}}$ such that $Q_{G}(S_i)\not =0$. Since
$Q_{G}(S_i)=Q_{G_1}(S_i)-Q_{G_2}(S_i)$ we have
$Q_{G_1}(S_i)\not=Q_{G_2}(S_i)$.
\end{proof}

\subsection{Algebraic View of the Model}\label{ss}
It is easy to show that for any hypergraph $G$ of rank $r$ the
adjacency $d$-dimensional matrix of $G$, $A_d^G$, for $d\ge r$, is
symmetric, contains a nonzero entry of dimension $r$ and
$$Q_G(S)=_p{A_d^G(x^S,x^S,\stackrel{d}{\ldots},x^S)}\stackrel{\Delta}{=} B_d^G(x^S).$$

For a symmetric $d$-dimensional matrix $A$ let
$B(x)=_pA(x,x,\stackrel{d}{\ldots},x)$ where $x\in\{0,1\}^n$. When
$x_1,\ldots,x_d\in \{0,1\}^n$ are pairwise disjoint the following
lemma shows that $A(x_1,\ldots,x_d)$ can be found by $2^d$ values
of~$B$.
\begin{lemma}\label{st}
If $x_1,\ldots,x_d\in \{0,1\}^n$ are pairwise disjoint then
$$A(x_1,\ldots,x_d)=_p\frac{1}{d!} \sum_{I\in 2^{[d]}}(-1)^{d-|I|}B\left(\sum_{i\in
I}x_i\right).$$
\end{lemma}
\begin{proof} Since
$$A(x_{1}+x'_{1},x_2,\ldots,x_d)=_pA(x_{1},x_2,\ldots,x_d)+A(x'_{1},x_2,\ldots,x_d)$$
and
$$A(x_{1},x_2,\ldots,x_d)=_pA(x_{\phi(1)},x_{\phi(2)},\ldots,x_{\phi(d)})$$
for any permutation $\phi$ on $[d]$, the result is analogous to the
fact that \begin{eqnarray}\label{y1yd} y_1y_2\cdots
y_d=_p\frac{1}{d!} \sum_{I\in 2^{[d]}}(-1)^{d-|I|}\left(\sum_{i\in
I}y_i\right)^d,
\end{eqnarray} for formal variables $y_1,\ldots,y_d$.
Now notice that
$$\left(\sum_{i\in I}y_i\right)^d=_p\sum_{ q_1+\cdots+q_d=d} {\chi}\left[\{i|q_i\not=0\}\subseteq I\right]{d\choose
q_1\ q_2\ \cdots q_d} y_1^{q_1}\cdots y_d^{q_d},$$ where
${\chi}[L]=1$ if the statement $L$ is true and $0$ otherwise.
Therefore, the coefficient of $y_1^{q_1}\cdots y_d^{q_d}$ in the
right hand side of (\ref{y1yd}) is
$$\sum_{I\in 2^{[d]}}(-1)^{d-|I|}{\chi}\left[\{i|q_i\not=0\}\subseteq
I\right]{d\choose q_1\ q_2\ \cdots q_d}$$
$$=_p{d\choose q_1\ q_2\ \cdots q_d}\sum_{I\in 2^{[d]}}(-1)^{d-|I|}{\chi}\left[\{i|q_i\not=0\}\subseteq
I\right].$$ Now if $\ell=|\{i|q_i\not=0\}|<d$ then
$$\sum_{I\in 2^{[d]}}(-1)^{d-|I|}{\chi}\left[\{i|q_i\not=0\}\subseteq
I\right]=_p\sum_{i=\ell}^{d}(-1)^{d-i}{d-\ell \choose
i-\ell}=_p\sum_{i=0}^{d-\ell}(-1)^{d-\ell-i}{d-\ell \choose i}=0.$$
If $\ell=|\{i|q_i\not=0\}|=d$ then $q_1=q_2=\cdots=q_d=1$ and
$$\sum_{I\in 2^{[d]}}(-1)^{d-|I|}{\chi}\left[\{i|q_i\not=0\}\subseteq
I\right]=_p1.$$ This implies the result.
\end{proof}

Let $G$ be a hypergraph of rank $d$ and $G^{(i)}$, $i\le d$, be the
sub-hypergraph of $G$ that contains all the edges in $G$ of size $i$
then
\begin{lemma}\label{lll}
If $x_1,\ldots,x_d\in \{0,1\}^n$ are pairwise disjoint then, we have
that $A^G_d(x_1,\ldots,x_d)= A^{G^{(d)}}_d(x_1,\ldots,x_d)$. In
particular, if $r(G)<d$ then $A^G_d(x_1,\ldots,x_d)=0$.
\end{lemma}
\begin{proof} Since $x_1,\ldots,x_d\in \{0,1\}^n$ are
pairwise disjoint we have
\begin{eqnarray*}
A_d^G(x_1,\ldots,x_d)&=&\sum_{i_1=1}^{n_1}\cdots
\sum_{i_d=1}^{n_d}
\frac{w^\star(\{i_1,i_2,\ldots,i_d\})}{N(d,|\{i_1,i_2,\ldots,i_d\}|)}x_{1i_1}\cdots
x_{di_d}\\
&=&\sum_{|\{i_1,\ldots,i_d\}|=d}\frac{w^\star(\{i_1,i_2,\ldots,i_d\})}{N(d,d)}x_{1i_1}\cdots
x_{di_d}\\
&=& A_d^{G^{(d)}}(x_1,\ldots,x_d).
\end{eqnarray*}
Now when $r(G)<d$ then $G^{(d)}$ is an independent set (has no
edges) and $A_d^{G^{(d)}}=0$. Then $$A_d^G(x)=A_d^{G^{(d)}}(x)=0.$$
\end{proof}

We now prove
\begin{lemma}\label{ZT}
Let $\Phi_d=\{z^{(d)}_1,\ldots,z^{(d)}_{k_d}\}\subset (\{0,1\}^n)^d$
where for every $i$ the vectors $z^{(d)}_{i,1},\ldots,z^{(d)}_{i,d}$
are pairwise disjoint. If $\Phi_d$ is a zero test set for
${\mathcal{A}}^\star_{d,(d!)m}$ then
$$S^{\Phi_d} \stackrel{\Delta}{=} \left\{ S^{y_J} \left|\  y_J=\sum_{j\in J} z^{(d)}_{i,j},\ J\subset [d]\begin{array}{c}\
\\ \, \\ \,
\end{array}\right.\right\}$$
is a detecting set for ${\mathcal{G}}^\star_{d,m}$.
\end{lemma}
\begin{proof} Let $\Phi_d$ be a zero test set for ${\mathcal{A}}^\star_{d,(d!)m}$. Let $G\in {\mathcal{G}}^\star_{d,m}$. Then
$A^G_d\not=0$ and $A^G_d\in {\mathcal{A}}^\star_{d,(d!)m}$.
Therefore, for every $G\in {\mathcal{G}}^\star_{d,m}$ there is
$z^{(d)}_i$ such that $A^G_d(z^{(d)}_i)\not=0.$ By Lemma \ref{st},
$$A^G_d(z^{(d)}_i)=_p\frac{1}{d!} \sum_{J\in 2^{[d]}}(-1)^{d-|J|}B^G_d\left(\sum_{j\in
J}z^{(d)}_{i,j}\right)\not=0,$$ and therefore for some $J_0\subset
[d]$,
$$B^G_d\left(\sum_{j\in J_0}z^{(d)}_{i,j}\right)\not=0,$$
which implies that $Q_G\left(S^{y_{J_0}}\right)\not=0$ for
$y_{J_0}=\sum_{j\in J_0} z^{(d)}_{i,j}.$
\end{proof}

We now show

\begin{lemma} \label{ZpR} A detecting set for ${\mathcal{G}}_{d,m}$
over $\Z_p$ is a detecting set for ${\mathcal{G}}_{d,m}$ over $\R$.
\end{lemma}
\begin{proof} Consider a detecting set
${\mathcal{S}}=\{S_1,S_2,\cdots,S_k\}\subseteq 2^V$ for
${\mathcal{G}}_{d,m}$ over $\Z_p$. Consider a $k\times q$
 matrix $M$ where $$q=\sum_{i=0}^d {n\choose i}$$ that its columns
are labelled with sets in $2^{[n]}$ of size at most $d$ and for
every $S\subset [n]$ of size at most $d$ we have $M[i,S]=1$ if
$S\subseteq S_i$ and $0$ otherwise. Consider for every graph $G\in
{\mathcal{G}}_{d,m}$ a $q$-vector $v_G$ that its entries are
labelled with subsets of $[n]$ of size at most $d$ and
$v_G[S]=w^\star(S)$. The labels in $v_G$ are in the same order as
the labels of the columns of $M$. Then it is easy to see that
$$Mv_G=_p (Q_G(S_1),\ldots,Q_G(S_k))^T.$$
Since $Mv_G\not=_p0$ for every $v_G\in \Z_p^{q}$ of weight at least
one and at most $m$, every $m$ columns in $M$ are linearly
independent over $\Z_p$. Since the entries of $M$ are zeros and ones
every $m$ columns in $M$ are linearly independent over $\R$.
Therefore, $$Mv_G= (Q_G(S_1),\ldots,Q_G(S_k))^T\not=0,$$ for every
$v_G\in \R^{q}$ of weight at least $1$ and at most $m$.
\end{proof}

\subsection{Distributions}\label{dist}
In this subsection we give a distribution that will be used in this
paper.

The {\it uniform disjoint distribution $\Omega_{d,n}(x)$} over
$(\{0,1\}^n)^d$ is defined as
$$
\Omega_{d,n}(x) = \left\{\begin{tabular}{cl} $\frac{1}{(d+1)^{n}}$
& $x_1,\ldots,x_d$ is pairwise disjoint.\\
0 & otherwise.
\end{tabular}
\right.
$$

In order to choose a random vector $x$ according to the uniform
disjoint distribution, one can randomly independently uniformly
choose $n$ elements $w_1,w_2,\ldots,w_n$ where $w_i\in[d+1]$ and
define the following vector
$x=(x_1,x_2,\ldots,x_d)\in(\{0,1\}^n)^d$:
$$
x_{ji} = \left\{\begin{tabular}{cl} 1 & $j=w_i$ and $w_i \in [d]$ \\
0 & otherwise.
\end{tabular}\right.
$$
We call any index $k\in[n]$ such that $ x_{jk} = 0 $ for all $j\in
[d]$ a \emph{free index}. Let $\Gamma_{d,n}\subset (\{0,1\}^n)^d$
be the set of all pairwise disjoint $d$-tuple.

\ignore{  \item Distribution $\Psi^{(i)}_{d,n}(x)$.\\
  The process of choosing $x = (x_1,x_2,\ldots,x_d)$ according to the
distribution
  $\Psi^{(i)}_{d,n}$ is the following. First we choose
  $$(x_1,x_2,\ldots,x_{i-1},x_{i+1},\ldots,x_{d})$$ according to the
  distribution $\Omega_{d-1,n}$. Then, we randomly uniformly choose
  $y\in\{0,1\}^n$. We redefine $x_{j}$ where $j\in[n]\setminus\{i\}$ to be
  $$
  x_j = x_j \wedge y
  $$
We will call an index $k\in [n]$ such that $y_i = 0$ a \emph{free
index}. Now, we are remained with choosing $x_i$. For any free index
$k$, we have $x_k = 1$ with probability 1/2 and $x_k = 0$ with
probability 1/2. For all other indices $\ell$ we have $x_\ell = 0$.
Note that $x$ is a $d$-disjoint tuple.}

\subsection{Preliminary Results} In this section we prove,

\begin{lemma}\label{eliminate}
Let $A\in \FF^{\times_d n}\setminus\{0^{\times_d n}\}$ be an
adjacency $d$-dimensional matrix of a hypergraph~$G$ of rank $d$.
Let $x = (x_1,x_2,\ldots,x_d)\in (\{0,1\}^n)^d$ be a randomly chosen
$d$-tuple, that is chosen according to the distribution
$\Omega_{d,n}$. Then
$$
\Pr_{x\in\Omega_{d,n}}[A(x)= 0]\leq 1 - \frac{1}{(d+1)^d}.
$$
\begin{proof}
Let $e=\{i_1,\ldots,i_d\}$ be an edge of size $|e|=d$ and let
$x'_j=(x_{j,i_1},\ldots,x_{j,i_d})$. Consider
$\phi(x'_1,\ldots,x'_d)$ that is equal to $A(x)$ with some fixed
$x_{j,i}=\xi_{j,i}\in \{0,1\}$ for $i\not\in e$. Since $A(x)$
contains the monomial $M=x_{1,i_1}x_{2,i_2}\cdots x_{d,i_d}$ and no
other monomial in $A(x)$ contains it, $\phi$ contains monomial $M$
and therefore $$\phi(x'_1,\ldots,x'_d)\not\equiv 0.$$ If we
substitute $x_{j_1,i_{j_2}}=0$ in~$\phi$ for all $j_1\not= j_2$ we
still get a nonzero function $\phi'(x_{1,i_1},x_{2,i_2},\cdots,
x_{d,i_d})$ that contains $M$. Therefore, there is
$\xi=(\xi_{1i_1},\xi_{2i_2},\cdots, \xi_{di_d})\in \{0,1\}^d$ such
that $\phi'(\xi)\not=0$. The probability that
$(x_{1,i_1},x_{2,i_2},\cdots, x_{d,i_d})=\xi$ and
$x_{j_1,i_{j_2}}=0$ for all $j_1\not= j_2$ is $({1}/{d+1})^d.$ This
implies the result.
\end{proof}

\end{lemma}

We will also use the following two lemmas from \cite{BM1,BM2}.
\begin{lemma}\label{PrimeElim}
Let $a\in \Z_p^n$ be a non-zero vector, where $p>wt(a)$ is a prime
number. Then for a uniformly randomly chosen vector
$x\in\{0,1\}^n$ we have
$$
\Pr_x[a^T x =_p 0]\leq \frac{1}{wt(a)^{\beta}},
$$
where $\beta = \frac{1}{2 + \log 3 }=0.278943\cdots$.
\end{lemma}

Let $\iota$ be a function on non-negative integers defined as
follows: $\iota(0)=1$ and $\iota(i)=i$ for $i>0$.

\begin{lemma} \label{inequality}
Let $m_1,m_2,\ldots,m_t$ be integers in $[m]\cup\{0\}$ such that
$m_1+m_2+\cdots+m_t=\ell\ge t.$ Then $\prod_{i=0}^t \iota(m_i)\ge
m^{\lfloor (\ell-t)/(m-1)\rfloor}.$
\end{lemma}

\section{Reconstructing Hypergraphs}
In this section we prove,

\begin{theorem}\label{Th1} There is a search set for ${\mathcal{G}}_{d,m}$ over
$\R$ of size $ k = O\left(\frac{m \log n}{\log m}\right). $
\end{theorem}

\begin{theorem}\label{Th2} There is a search set for ${\mathcal{G}'}_{d,m}$ over
$\R$ of size $ k = O\left(\frac{m \log \frac{n^d}{m}}{\log
m}\right)$, where ${\mathcal{G}'}_{d,m}$ denotes the set of all
weighted hypergraphs over $\R$ of rank at most $d$, at most $m$
edges and weights that are integers bounded by $w = poly(n^d/m)$.
\end{theorem}

\begin{proof} We give the proof of Theorem \ref{Th1}.
The proof of Theorem \ref{Th2} is similar. More details in the full
paper.

Let $m<p<2m$ be a prime number. Suppose there is a zero test set
from $\Gamma_{d,n}$ for ${\mathcal{A}}^\star_{d,m}$ over $\Z_p$ of
size $T(n,m,d)$. By Lemma \ref{ZT}, there is a detecting set for
${\mathcal{G}}^\star_{d,m}$ over $\Z_p$ of size $2^dT(n,(d!)m,d)$.
Therefore, by Lemma \ref{lll}, there is a detecting set for
${\mathcal{G}}_{d,m}$ over $\Z_p$ of size
$T'(n,m,d)=\sum_{\ell=1}^d2^\ell T(n,(\ell!)m,\ell).$ By Lemma
\ref{ZpR}, there is a detecting set for ${\mathcal{G}}_{d,m}$ over
$\R$ of size $T'(n,m,d)$. Finally, by Lemma \ref{dtos}, there is a
search set for ${\mathcal{G}}_{d,m}$ over $\R$ of size $T'(n,2m,d)$.
Now for constant $d$, if
\begin{eqnarray}\label{enough}
T(n,m,d)=O\left(\frac{m \log n}{\log m}\right),
\end{eqnarray} then $T'(n,2m,d)=O(T(n,m,d))$. Therefore it is
enough to prove the following.

\begin{lemma}\label{l9}
Let $p$ be a prime number such that $m<p<2m$. There exists a set
$S=\{x_1 , x_2 ,\ldots, x_k \}\subseteq (\{0,1\}^n)^d $ where
$x_i=(x_{i,1},\ldots,x_{i,d})\in \Gamma_{d,n}$ for $i\in[k]$ and
$$
k = O\left(\frac{m \log n}{\log m}\right),
$$
such that: for every $d$-dimensional matrix $A\in\Z_p^{\times_d n}
\setminus \{0^{\times_d n}\}$ with $1\le wt_d(A)\leq m$ there
exists an $i$ such that
$
A(x_i) \neq_p 0.
$
\end{lemma}
\begin{proof}
Since $wt_d(A)>1$ the matrix $A$ has at least one nonzero entry of
dimension $d$. We will assume that all the entries of dimension
less than $d$ are zero, that is, $wt(A)=wt_d(A)$. This is because,
by Lemma \ref{lll}, the entries of dimension less than $d$ have no
effect when the vectors $x_i\in \Gamma_{d,n}$.

We divide the set of such matrices $ \mathcal{A} = \{A\, | \,
A\in\Z_p^{\times_d n}\setminus \{0^{\times_d n}\}\ \mathrm{and}\
wt(A)\leq m \} $ into $d+1$ (non-disjoint) sets:
\begin{itemize}
  \item $\mathcal{A}_0$: The set of all non-zero matrices $A\in \Z_p^{\times_d n}$ such
that $wt(A)\leq m/\log m$.
  \item $\mathcal{A}_j$ for $ j = 1,\ldots,d$: The set of all non-zero matrices
$A\in\Z_p^{\times_d n}$ such that $m\geq wt(A) > m/\log m$ and
there are at least
  $$
  \left(\frac{m}{\log m}\right)^{1/d}
  $$
  non-zero elements in
  $
  I_j = \{i_j |\exists (i_1,i_2,\ldots,i_{j-1},i_{j+1},\ldots,i_d):
A_{i_1,i_2,\ldots,i_d}\neq 0 \}.
  $
\end{itemize}
Note that $I =
\{(i_1,i_2,\ldots,i_d)|A_{i_1,i_2,\ldots,i_d}\not=0\}\subseteq
I_1\times I_2\times \cdots \times I_d$ and therefore either $I=
wt(A)\leq m/\log m$ or there is $j$ such that $|I_j|>(m/\log
m)^{1/d}$. Therefore, $\mathcal{A} = \mathcal{A}_0 \cup
\mathcal{A}_1\cup \cdots\cup \mathcal{A}_d$.

Using the probabilistic method, we give $d+1$ sets of pairwise
disjoint tuples of vectors $S_0,S_1,\ldots,S_d$ such that for every
$j\in\{0\}\cup [d]$ and $A\in\mathcal{A}_j$ there exists a $d$-tuple
$x$ in $S_j$ such that $A(x)\neq 0$ and
$$
|S_0| + |S_1| + \cdots + |S_d| = O\left(\frac{m\log n}{\log m}\right).
$$

\noindent \textbf{Case 1}: $A\in \mathcal{A}_0$: For a random
$d$-tuple $x$, chosen according to the distribution $\Omega_{d,n}$
we have that
$$
\Pr_x [A(x) =_p 0 ] \leq 1 - \frac{1}{(d+1)^d}.
$$
If we randomly choose $$k_1=\frac{cm\log n}{\log m}$$ $d$-tuples,
$x_1,\ldots,x_{k_1}$, according to the distribution
$\Omega_{d,n}$, then the probability that $A(x_i)= 0$ for all
$i\in[k_1]$ is
$$
\Pr [\forall i\in[k_1]: A(x_i) =_p 0 ] \leq \left(1 -
\frac{1}{(d+1)^d}\right)^{k_1}.
$$
Therefore, by union bound, the probability that there exists a
matrix $A\in\mathcal{A}_0$ such that $A(x_i)= 0$ for all $i\in[k_1]$
is
\begin{eqnarray*}
\Pr [\exists A\in\mathcal{A}_0, \forall i\in[k_1]: A(x_i) =_p 0 ]
&\leq& \binom{n^d}{\frac{m}{\log m}}p^{\frac{m}{\log m}}\left(1 -
\frac{1}{(d+1)^d}\right)^{\frac{cm\log n}{\log m}}\\ &<&
n^{d\frac{m}{\log m}}n^{\frac{m}{\log m}}n^{-\frac{c'cm}{\log m}}
<1,
\end{eqnarray*}
for some constant $c$. This implies the result.
$$ $$

\noindent \textbf{Case2}: $A\in \mathcal{A}_j$ where $j=1,\ldots,d$:
We will assume w.l.o.g that $j=1$. We first prove the following
lemma
\begin{lemma}\label{smalldim}
Let $U\subseteq \Z_p^{\times_{d-1}n}$ be the set of all
$d-1$-dimensional matrices with weight smaller than $m^{d/(d+1)}$.
For $A\in U$ let $\Upsilon(A)\subseteq[n]$ be following set
$$
\Upsilon(A)=\{j\, | \,\exists A_{i_1,i_2,\ldots,i_{d-1}}\neq 0 \
\mathrm{and}\  j\not\in \{i_1,i_2,\ldots,i_{d-1}\} \}.
$$
Define $ Q = \{(A,j)\, | \, A\in U \ \mathrm{and}\ j\in\Upsilon(A)
\}. $ Then, there is a constant $c_0$ such that for every $C>c_0$
and
$$
k_2 = C\frac{m\log n}{\log m}
$$
there exists a multi-set of $d-1$-tuples of (0,1)-vectors $Z=
\{z_1,z_2,\ldots,z_{k_2}\}\subseteq (\{0,1\}^n)^{d-1}$ such that
for every $(A,j)\in Q$ the size of the set
$$
Z_{(A,j)} = \{ i\, | A(z_i) \neq 0 \ \mathrm{and}\ j\ \mathrm{is\ a\
free\ index} \}
$$
is at least $\frac{k_2}{2d^{d}}.$
\end{lemma}

\begin{proof}
Let $z_i = (z_{i,1},z_{i,2},\ldots,z_{i,d-1})\in(\{0,1\}^n)^{d-1}$
be random $d-1$-tuple of $(0,1)$-vector  chosen according to the
distribution $\Omega_{d-1,n}$. For $(A,j)\in Q$, and by Lemma
\ref{eliminate}, we have
\begin{eqnarray*}
\Pr_{z_i\in \Omega_{d-1,n}}[A(z_i)\neq 0\ \mathrm{and}\ j \
\mathrm{is\ a\ free}] = \Pr[j \ \mathrm{is\ free}]\Pr[A(z_i)\neq
0|  j\mathrm{\ is\ free}] \geq
\frac{1}{d}\cdot\frac{1}{d^{d-1}}=\frac{1}{d^d}.
\end{eqnarray*}

Therefore, the expected size of $Z_{(A,j)}$ is greater than
$\frac{k_2}{d^d}$. By Chernoff bound, if we randomly choose all
$z_i$, $i\in[k_2]$ according to the distribution $\Omega_{d-1,n}$,
then, we have
$$
\Pr\left[|Z_{(A,j)}|\leq \frac{k_2}{2d^d}\right]\leq
e^{\frac{-k_2}{8 d^d }}.
$$
Thus, the probability that there exists $(A,j)\in Q$ such that
$|Z_{(A,j)}| \leq \frac{k_2}{2d^d}$ is
\begin{eqnarray*}
\Pr\left[\exists (A,j)\in Q: |Z_{(A,j)}|\leq
\frac{k_2}{2d^d}\right]&\leq& \frac{|Q|}{e^{\frac{-k_2}{8 d^d }}}
\leq \frac{|U \times [n]|}{e^\frac{-k_2}{8 d^d }} \leq
\frac{n\binom{n^{d-1}}{m^{d/(d+1)}}p^{m^{d/(d+1)}}}{e^{\frac{Cm\log
n}{8d^d\log m}}}\\
&\le&\frac{n\binom{n^{d-1}}{m^{d/(d+1)}}n^{m^{d/(d+1)}}}{n^{\frac{C(\log
e)m}{8d^d\log m}}}
\le\frac{n^{O\left(m^{d/(d+1)}\right)}}{n^{\frac{Cc'm}{\log m}}}
 <1,
\end{eqnarray*}
for large enough $C$. This implies the result.
\end{proof}
Now, Let $U$ and $Q$ be the sets we defined in Lemma \ref{smalldim}.
Let $A\in\mathcal{A}_1$. Since $wt(A)\leq m$ there are at most
$m^{1/(d+1)}$ $d-1$-dimensional matrices
$(A_{i_1,i_2,\ldots,i_d})_{i_1=j,i_2,\ldots,i_d}$ with weight
greater than $m^{d/(d+1)}$. Therefore, there is at least
$$
q = \left(\frac{m}{\log m}\right)^{1/d} - m^{1/(d+1)}
$$
indices $j$ such that
$(A_{i_1,i_2,\ldots,i_d})_{i_1=j,i_2,\ldots,i_d} \in U$. Let $U'$
contain any $q$ indices such that
$(A_{i_1,i_2,\ldots,i_d})_{i_1=j,i_2,\ldots,i_d} \in U$. Let $A_U$
be the matrix $$(A_{i_1,i_2,\ldots,i_d})_{i_1\in U',i_2,\ldots,i_d}
.$$ Let $z_1,z_2,\ldots,z_{k_2}\in (\{0,1\}^n)^{d-1}$ be the set we
proved its existence in Lemma~\ref{smalldim}. We now choose $x_i\in
\{0,1\}^n$, $i\in[k_2]$ in the following way: Take $z_i$. For every
free index $j$, choose $x_{ij}$ to be ``1'' with probability 1/2 and
``0'' with probability 1/2 (independently for every~$j$). All other
entries in $x_i$ are zero, that is, all entries that correspond to
non-free index $j$ in $z_i$ are zero. Let $u\in \{0,1\}^n$ be a
vector where $u_j = 1$ if $j\in U'$ and zero otherwise. Also, for a
$d-1$-tuple $z_i$ let $v_i\in \{0,1\}^n$ be the vector where
$v_{ij}=1$ if $j$ is a free index in $z_i$ and $v_{ij}=0$ otherwise.
By Lemma \ref{PrimeElim} we have that
\begin{eqnarray}\label{q2}
\Pr_x [A(x_i,z_i)=_p 0] \leq  \prod_{i}\frac{1}{\iota(wt(v_i
* A(\cdot, z_i)))^{\beta}}
 \leq  \prod_{i}\frac{1}{\iota(wt(v_i
* (u * A(\cdot, z_i))))^{\beta}}.
\end{eqnarray}

Note that, $A$ is a hypergraph, thus, for every $j$ such that $
(A_{i_1,i_2,\ldots,i_d})_{i_1=j,i_2,\ldots,i_d} \in U$, we have that
$$((A_{i_1,i_2,\ldots,i_d})_{i_1=j,i_2,\ldots,i_d}, j)\in Q.$$
Therefore,
$$
\sum_i wt(v_i
* (u * A(\cdot, z_i))) \geq \frac{qk_2}{2d^d}.
$$
Using Lemma \ref{inequality} we have
$$
\prod_{i}{\iota(wt(v_i
* (u * A(\cdot, z_i))))} \geq q^{\lfloor\frac{\frac{qk_2}{2d^d} - k_2}{q-1}\rfloor} =
m^{c_1k_2}.
$$
Therefore, using (\ref{q2}), $\Pr_x [A(x_i,z_i) =_p 0] \leq
\frac{1}{m^{c_1\beta k_2}} $. Thus, the probability that there
exists a matrix $A\in\mathcal{A}_1$ such that for all $i\in[k_2]$ we
have $A(x_i,z_i) = 0$ is
$$
\Pr_x [A(x_i,z_i) =_p 0]  \leq \frac{|\mathcal{A}_1|}{m^{c_1\beta
k_2}} \leq \frac{\binom{n^d}{m}p^m}{m^{c_1\beta k_2}} \le
\frac{n^{dm}n^m}{m^{c_1\beta k_2}} < 1,
$$
for large enough constant. This implies Lemma \ref{l9}.
\end{proof}
This completes the proof of Theorem \ref{Th1}.
\end{proof}

\end{document}